\documentclass[11pt]{article}

\usepackage[margin=1in]{geometry}
\usepackage{graphicx}

\usepackage{amssymb}
\usepackage{latexsym}
\usepackage{amsthm}
\usepackage{amsmath}
\usepackage{amsfonts}
\usepackage{url}
\usepackage{graphicx}
\usepackage{bbm}

\usepackage{natbib}
\usepackage{epstopdf}

\usepackage[utf8]{inputenc}   %
\usepackage[T1]{fontenc}      %
\usepackage{hyperref}           %
\usepackage{url}                    %
\usepackage{booktabs}         %
\usepackage{amsfonts}         %
\usepackage{nicefrac}           %
\usepackage{microtype}        %
\usepackage{lmodern}

\newtheorem{claim}{Claim}[section]
\newtheorem{theorem}[claim]{Theorem}

\newtheorem{lemma}[claim]{Lemma}

\begin{document}
\title{Learning convex polyhedra with margin\footnote{An extended abstract of this work appeared as \emph{Learning convex polytopes with margin} in the proceedings of NIPS2018.}}

\author{Lee-Ad Gottlieb\thanks{\texttt{leead@ariel.ac.il} Ariel University, Israel.}         \and
        Eran Kaufman\thanks{\texttt{erankfmn@gmail.com} Ariel University, Israel.}             \and
        Aryeh Kontorovich\thanks{\texttt{karyeh@cs.bgu.ac.il} Ben-Gurion University, Israel.}	\and
	Gabriel Nivasch\thanks{\texttt{gabrieln@ariel.ac.il} Ariel University, Israel.}
}

\date{November 2, 2021}

\maketitle
\begin{abstract}%
  We present an improved algorithm for {\em quasi-properly} learning
  convex polyhedra in the realizable PAC setting
  from data with a margin.
  Our learning algorithm constructs a consistent polyhedron
  as an intersection of about $t \log t$ halfspaces with constant-size margins in time polynomial in $t$
  (where $t$ is the number of halfspaces forming an optimal polyhedron).

  We also identify distinct
  generalizations of the notion of margin from hyperplanes to polyhedra
  and investigate how they relate geometrically;
  this result may 
  have ramifications
  beyond the learning setting.
	
\end{abstract}

\newcommand{\pr}[1]{\P\!\paren{#1}}
\renewcommand{\P}{\mathbb{P}}
\newcommand{\E}{\mathbb{E}}
\newcommand{\tru}[1]{{#1}^\star}
\newcommand{\emp}[1]{\hat{#1}}
\newcommand{\F}{\mathcal{F}}
\renewcommand{\H}{\calH}
\newcommand{\X}{\calX}
\newcommand{\C}{\calC}
\newcommand{\Y}{\calY}
\newcommand{\bs}{\boldsymbol}
\renewcommand{\vec}[1]{\bs{\mathrm{#1}}}
\newcommand{\AK}[1]{{\bf[AK: #1]}}
\newcommand{\hide}[1]{}
\newcommand{\oo}[1]{\frac{1}{#1}}
\newcommand{\chr}{\boldsymbol{\mathbbm{1}}} %
\newcommand{\pred}[1]{\chr_{\left\{ #1 \right\}}}
\newcommand{\nrm}[1]{\left\Vert #1 \right\Vert}
\newcommand{\trn}{^\intercal} %
\newcommand{\bx}{\vec{x}}
\newcommand{\bw}{\vec{w}}
\newcommand{\bX}{\vec{X}}
\newcommand{\inv}{^{-1}} %
\newcommand{\calA}{\mathcal{A}}
\newcommand{\calB}{\mathcal{B}}
\newcommand{\calC}{\mathcal{C}}
\newcommand{\calD}{\mathcal{D}}
\newcommand{\calE}{\mathcal{E}}
\newcommand{\calF}{\mathcal{F}}
\newcommand{\calG}{\mathcal{G}}
\newcommand{\calH}{\mathcal{H}}
\newcommand{\calI}{\mathcal{I}}
\newcommand{\calK}{\mathcal{K}}
\newcommand{\calL}{\mathcal{L}}
\newcommand{\calM}{\mathcal{M}}
\newcommand{\calN}{\mathcal{N}}
\newcommand{\calP}{\mathcal{P}}
\newcommand{\calS}{\mathcal{S}}
\newcommand{\calT}{\mathcal{T}}
\newcommand{\calV}{\mathcal{V}}
\newcommand{\calX}{\mathcal{X}}
\newcommand{\calY}{\mathcal{Y}}
\newcommand{\calZ}{\mathcal{Z}}
\newcommand{\tha}{\theta}
\newcommand{\R}{\mathbb{R}}
\newcommand{\N}{\mathbb{N}}
\newcommand{\Z}{\mathbb{Z}}
\newcommand{\Q}{\mathbb{Q}}
\newcommand{\beq}{\begin{eqnarray*}}
\newcommand{\eeq}{\end{eqnarray*}}
\newcommand{\beqn}{\begin{eqnarray}}
\newcommand{\eeqn}{\end{eqnarray}}
\newcommand{\paren}[1]{\left( #1 \right)}
\newcommand{\sqprn}[1]{\left[ #1 \right]}
\newcommand{\tlprn}[1]{\left\{ #1 \right\}}
\newcommand{\set}[1]{\tlprn{#1}}
\newcommand{\abs}[1]{\left| #1 \right|}
\newcommand{\ceil}[1]{\ensuremath{\left\lceil#1\right\rceil}}
\newcommand{\floor}[1]{\ensuremath{\left\lfloor#1\right\rfloor}}
\newcommand{\gn}{\, | \,}
\renewcommand{\th}{\ensuremath{^{\mathrm{th}}}~}
\def\longto{\mathop{\longrightarrow}\limits}
\newcommand{\ninf}{\longto_{n\to\infty}}
\newcommand{\tinf}{\longto_{t\to\infty}}
\newcommand{\eps}{\varepsilon}
\newcommand{\trace}{\operatorname{tr}}
\newcommand{\gerr}{\operatorname{err}}
\newcommand{\serr}{\widehat{\gerr}}
\newcommand{\sign}{\operatorname{sign}}

\newcommand{\conv}{\operatorname{conv}}
\newcommand{\poly}{\operatorname{poly}}
\newcommand{\fat}{\operatorname{fat}}

\newcommand{\opt}{\mathrm{opt}}
\newcommand{\VC}{{\textrm{{\tiny \textup{VC}}}}}

\newcommand{\pwrset}{\mathcal P}
\newcommand{\scp}[2]{\left\langle #1, #2 \right\rangle}
\newcommand{\haus}[2]{\delta{\left(#1,#2\right)}}
\newcommand{\hausS}[1]{\widehat\delta_{#1}}
\newcommand{\interval}[2]{\left\llbracket #1,#2\right\rrbracket}

\section{Introduction}

In the theoretical PAC learning setting \citep{DBLP:journals/cacm/Valiant84},
one considers an abstract {\em instance space} $\calX$ --- which, most commonly,
is either the Boolean cube $\set{0,1}^d$ or the Euclidean space $\R^d$.
For the former setting, an extensive literature has explored
the statistical and computational aspects of learning Boolean functions
\citep{DBLP:conf/stoc/Angluin92,DBLP:journals/tcs/HellersteinS07}.
Yet for the Euclidean setting, a corresponding theory of learning geometric concepts
is still being actively developed
\citep{DBLP:journals/algorithmica/KwekP98,DBLP:journals/jcss/JainK03,DBLP:conf/colt/AndersonGR13,DBLP:conf/colt/KaneKM13}.
The focus of this paper is the latter setting.

The simplest nontrivial geometric concept is perhaps the halfspace.
These concepts are well-known to be hard to agnostically learn
\citep{DBLP:journals/jcss/HoffgenSH95}
or even approximate
\citep{AMALDI1995181, AMALDI1998237, DBLP:journals/jcss/Ben-DavidEL03}.
Even the realizable case, while commonly described as ``solved'' via the Perceptron algorithm
or linear programming (LP), is not straightforward: The Perceptron's runtime
is quadratic in the inverse-margin, 
while solving the consistent hyperplane problem in
strongly polynomial time 
is equivalent to solving 
the general LP problem in
strongly polynomial time
\citep{Nikolov18, Chvatal18} (we include the proof in Appendix~\ref{app} for completeness). The problem of solving LP in strongly polynomial time has been open for decades \citep{DBLP:conf/innovations/BaraszV10}.
Thus, an unconditional (i.e., infinite-precision and independent of data configuration
in space)
polynomial-time solution for the consistent hyperplane problem hinges on the
strongly polynomial LP conjecture.

If we consider not a single halfspace, 
but polyhedra defined by the intersection of multiple halfspaces,
the computational and generalization bounds
rapidly become more pessimistic.
\citet{Megiddo1988}
showed that the problem of deciding
whether two sets of points in general
space can be separated by the intersection of two hyperplanes
is $\mathrm{NP}$-complete, and 
\citet{DBLP:journals/jcss/KhotS11}
showed that ``unless $\mathrm{NP}=\mathrm{RP}$, it is hard to (even) weakly
PAC-learn intersection of two halfspaces'',
even when allowed the richer class of $O(1)$ intersecting halfspaces.
Under cryptographic assumptions,
\citet{DBLP:journals/jcss/KlivansS09}
showed that learning an intersection of $n^\eps$
halfspaces is intractable regardless of hypothesis representation.

Since the margin assumption is what allows one to
find a consistent hyperplane in
provably strongly polynomial time,
it is natural to seek to generalize this scheme
to intersections of $t$ halfspaces each with margin $\gamma$;
we call this the $\gamma$-{\em margin} of a $t$-polyhedron.
This problem was considered by 
\citet{DBLP:journals/ml/ArriagaV06}, who showed that such
a polyhedron can be learned (in dimension $d$) in time
$$O(dmt) +
(t \log t)^{O( (t/\gamma^2) \log (t/\gamma))}$$
with sample complexity
$m = O\paren{(t/\gamma^2) \log (t) \log(t/\gamma)}$
(where we have taken the PAC-learning parameters $\eps,\delta$ to be constants).
In fact, they actually construct a candidate $t$-polyhedron 
as their learner; as such, their approach is an example of 
{\em proper learning}, where the hypothesis is chosen from the same
concept class as the true concept.
In contrast, 
\citet{DBLP:journals/jcss/KlivansS08}
showed that a $\gamma$-margin $t$-polyhedron
can be learned by constructing a function that approximates the polyhedron's
behavior, without actually constructing a $t$-polyhedron.
This is an example of {\em improper learning}, where the hypothesis
is selected from a broader class than that of the true concept.
They achieved a runtime of
$$
\min\set{
d (t/\gamma)^{O(t \log t \log(1/\gamma))},
d\paren{\frac{\log t}{\gamma}}^{O(\sqrt{1/\gamma} \log t)}
}
$$
and sample complexity
$m = O\paren{(1/\gamma)^{t \log t \log (1/\gamma)}}.$
Very recently, \citet{GK18} improved on this latter result, 
constructing a function hypothesis in time
$\poly(d,t^{O(1/\gamma)})$,
with sample complexity exponential in $\gamma^{-1/2}$.
Our notion of learning, termed here {\em quasi-proper},
lies somewhere between proper and improper learning in terms of stringency.
Quasi-proper learning requires an infinite nested family
of hypotheses of growing complexity: $H_1\subseteq H_2\subseteq\ldots$.
A learning algorithm is said to be quasi-proper for a class $H_k$ if it
returns an $h\in H_\ell$ for $\ell\ge k$.
As just stated,
the
definition does not formally distinguish quasi-proper and improper learning.
The former is a stricter notion because we require the $\set{H_k}$ to all belong
to the same parametric class, such as all $t$-facet polyhedra in $\R^d$,
for $t+d\le k$. This notion could be formalized beyond the case of polyhedra,
but since this paper only deals with learning the latter, such a level of generality
would be out of scope.

\paragraph{Our results.}
The central contribution of the paper is
improved algorithmic runtimes and sample complexity
for computing separating polyhedra (Theorem \ref{thm:compute}).
In contrast to the algorithm of 
\citet{DBLP:journals/ml/ArriagaV06},
whose runtime is exponential in $t/\gamma^2$, 
and to that of \citep{GK18}, whose sample complexity
is exponential in $\gamma^{-1/2}$,
we give an algorithm with polynomial sample complexity
$m = \tilde{O}(t/\gamma^2)$
and runtime only
$m^{\tilde{O}(1/\gamma^2)}$.
We accomplish this by constructing an 
$O(t \log m)$-polyhedron that correctly separates the data.
This means that our hypothesis is drawn from a broader class than
the $t$-polyhedra of 
\citet{DBLP:journals/ml/ArriagaV06}
(allowing faster runtime),
but from a much narrower class than the functions of
\citet{DBLP:journals/jcss/KlivansS08, GK18}
(allowing for improved sample complexity).

Complementing our algorithm, we provide the first nearly
matching hardness-of-approximation bounds, which (roughly) demonstrate
that an exponential dependence on $t\gamma^{-2}$ is unavoidable
for the computation of separating $t$-polyhedra, 
under standard complexity-theoretic assumptions
(Theorem \ref{thm:gc-reduction}). This motivates our
consideration of $O(t \log m)$-polyhedra instead.

Our final contribution is in introducing a new 
and intuitive notion of polyhedron margin:
This is the $\gamma$-{\em envelope} of a convex polyhedron,
defined as all points within distance $\gamma$ of the polyhedron's boundary,
as opposed to the above $\gamma$-{\em margin} of the polyhedron, 
defined as the intersection of the $\gamma$-margins of the 
hyperplanes forming the polyhedron. 
(See Figure~\ref{fig_margin_env}
for an illustration,
and Section \ref{sec:prelim} for precise definitions.)
Note that these two objects may exhibit vastly different
behaviors, particularly at a sharp intersection of two
or more hyperplanes.
It seems to us that the envelope of a polyhedron is a more natural structure than its margin:
Indeed, taking an envelope has the effect of rounding the corners of the polyhedron, and rounded polyhedra occur in nature and have been the subject of  mathematical study
\citep{Onaka05,Onaka08,Andersson08,BB17}.
Yet we find the margin
more amenable to the derivation of 
combinatorial dimension bounds
(Theorem~\ref{thm:poly-fat})
and algorithms
(Theorem \ref{thm:compute}).
We demonstrate that
results derived for margins can be adapted to apply to envelopes as well.
We prove that when confined to the unit ball, 
under natural conditions the
$\gamma$-envelope fully contains within it
the $(\gamma^2/2)$-margin (Theorem \ref{thm:lem_in_ball}),
and this implies that statistical and algorithmic results for
the latter hold for the former as well.
Using the same techniques, we also derive a related result of independent
interest concerning expanding polyhedra (Section~\ref{sec:hspeed}).
In Section~\ref{sec:exp} we present some simulation results.

\paragraph{Related work.}
When general convex bodies are considered under the uniform distribution\footnote{
  Since the concept class of convex sets has infinite VC-dimension,
  without distribution assumptions, an adversarial distribution can require
  an arbitrarily large sample size, even in $2$ dimensions \citep{Kearns97}.
  }
(over the unit ball or cube),
exponential (in dimension and accuracy)
sample-complexity bounds were obtained
by \citet{DBLP:conf/colt/RademacherG09}.
This may motivate the consideration of convex polyhedra, and indeed
a number of works have studied the problem of learning convex polyhedra, including
\citet{DBLP:conf/colt/Hegedus94,DBLP:journals/algorithmica/KwekP98,
  DBLP:conf/colt/AndersonGR13,DBLP:conf/colt/KaneKM13,DBLP:conf/nips/KantchelianTHBJT14}.
\citet{DBLP:conf/colt/Hegedus94}
examines query-based exact identification of
convex polyhedra with integer vertices,
with runtime polynomial in 
the number of vertices 
(note that the number of vertices can be exponential in the number of facets
\citep{MR1899299}).
\citet{DBLP:journals/algorithmica/KwekP98} also rely on membership queries
(see also references therein regarding prior results,
as well as strong positive results in
$2$ dimensions).
\citet{DBLP:conf/colt/AndersonGR13} efficiently approximately recover an unknown simplex
from uniform samples inside it.
\citet{DBLP:conf/colt/KaneKM13} learn halfspaces
under the log-concave distributional assumption.

The recent work of
\citet{DBLP:conf/nips/KantchelianTHBJT14}
bears a superficial resemblance to ours, but the two are actually
not directly comparable.
What they term {\em worst case margin} will indeed correspond to our {\em margin}.
However, their optimization problem is non-convex, and the solution relies on heuristics
without rigorous run-time guarantees.
Their generalization bounds exhibit a better dependence on the number $t$ of halfspaces
than our Theorem~\ref{thm:poly-fat}
($O(\sqrt t)$ vs. our
$O(t\log t)$). However, the hinge loss appearing in their Rademacher-based bound could be significantly
worse than the 0-1 error appearing in our VC-based bound.
We stress, however, that the main contribution of our paper is algorithmic rather than statistical.

\section{Preliminaries}
\label{sec:prelim}

\paragraph{Notation.}
For $\vec x\in\R^d$, 
we denote its Euclidean norm $\nrm{\vec x}_2:=\sqrt{\sum_{i=1}^d \vec x(i)^2}$ by $\nrm{\vec x}$,
and for $n\in\N$, we write $[n]:=\set{1,\ldots,n}$.
Our {\bf instance space} $\calX$ is the unit ball in $\R^d$: $\calX=\set{\vec x\in\R^d:\nrm{\vec x}\le1}$.
We assume familiarity with the notion of VC-dimension
as well as with basic PAC definitions such as {\em generalization error}
(see, e.g., \citet{Kearns97}).

\paragraph{Polyhedra.}
A (convex) polyhedron $P\subset\R^d$ is the intersection of a finite number $t$ of closed halfspaces. Each halfspace is bounded by a hyperplane $(\vec w_i,b_i)\in\R^d\times \R$ with $\nrm{\vec w_i}=1$ for each $i$, and $P$ is given by:
\beqn
\label{eq:Pdef}
P=\set{\vec x\in\R^d: \min_{i\in[t]} \vec w_i\cdot \vec x+b_i\ge0}.
\eeqn
(A polyhedron is not necessarily bounded. A bounded polyhedron is called a polytope.)
A hyperplane $(\vec w,b)$ is said to classify a point $\vec x$ as positive (resp., negative)
with margin $\gamma$ if $\vec w\cdot\vec x+b\ge\gamma$ (resp., $\le-\gamma$).
Since $\nrm{\vec w}=1$, this means that $\vec x$ is $\gamma$-far from the hyperplane
$\set{\vec x'\in\R^d: \vec w\cdot\vec x'+b=0}$, in $\ell_2$ distance.

\paragraph{Margins and envelopes.}
We consider two natural ways of extending this notion to polyhedra: the $\gamma$-margin and the $\gamma$-envelope.
For a polyhedron defined by $t$ hyperplanes as in (\ref{eq:Pdef}),
we say that $\vec x \in P$
is in the {\em inner $\gamma$-margin} of $P$ if 
$$0\le\min_{i\in[t]} \vec w_i\cdot \vec x+b_i\le\gamma$$
and that
$\vec x \notin P$
is in the {\em outer $\gamma$-margin} of $P$ if 
$$0 > \min_{i\in[t]} \vec w_i\cdot \vec x+b_i\ge-\gamma.$$
Similarly, we say that
$\vec x$
is in the {\em outer $\gamma$-envelope} of $P$ if $\vec x\notin P$ and $\inf_{\vec p\in P}\nrm{\vec 
x-\vec p}\le\gamma$
and that
$\vec x$
is in the {\em inner $\gamma$-envelope} of $P$ if $\vec x\in P$ and $\inf_{\vec p\notin P}\nrm{\vec x-\vec p}\le\gamma$.

We call the union of the inner and the outer $\gamma$-margins the {\em $\gamma$-margin}, and we denote it by $\partial P^{[\gamma]}$. Similarly, we call the union of the inner and the outer $\gamma$-envelopes the {\em $\gamma$-envelope}, and we denote it by $\partial P^{(\gamma)}$.

\begin{figure}
\centerline{\includegraphics{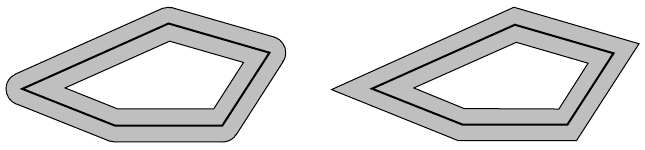}}
\caption{\label{fig_margin_env} The $\gamma$-envelope $\partial P^{(\gamma)}$ (left) and $\gamma$-margin $\partial P^{[\gamma]}$ (right) of a polyhedron $P$.}
\end{figure}

The two notions are illustrated in
Figure~\ref{fig_margin_env}.
As we show in Section~\ref{sec:nivasch} below, the inner envelope coincides with the inner margin,
but this is not the case for the outer objects: The outer margin always
contains the outer envelope, and could be of arbitrarily
larger volume.

\newcommand{\tscale}{\theta}

\paragraph{Fat-shattering dimension.}
Let $\X$ be a set and $\F\subset\R^\X$. For $\tscale>0$,
a set $S=\set{x_1,\ldots,x_m}\subset\X$
is $\tscale$-shattered by
$\F$
\beq
\sup_{r\in\R^m}
\;
\min_{y\in\set{-1,1}^m}
\;
\sup_{f\in\F}
\;
\min_{i\in[m]}
\;
y_i(f(x_i)-r_i)\ge \tscale.
\eeq
The $\tscale$-fat-shattering dimension, denoted by $\fat_\tscale(\F)$,
is the size of the largest $\tscale$-shattered set (possibly $\infty$).

\paragraph{Fat hyperplanes and polyhedra.}
Binary classification requires a collection of {\em concepts} mapping the instance space (in our case,
the unit ball in
$\R^d$)
to $\set{-1,1}$.
Any real-valued function class can be converted into a binary concept class by
composing with the sign function.
The generalization error of such a predictor can be controlled
by the fat-shattering dimension of the function class:
\begin{theorem}[\citet{299098} Theorem 4.5]
\label{thm:bst}
Let $\F$ be a collection of real-valued functions over some set $\X$,
define $D=\fat_{1/16}(\F)$, and let $P$ be some probability distribution
on $\X\times\set{-1,1}$. Suppose that $(x_i,y_i)$, $i\in[n]$, are drawn
independently according to $P$, and that some $f\in\F$ achieves 
{\em strong separation} in the sense that
\beqn
\label{eq:strong-sep}
\min_{i\in[n]} y_if(x_i)\ge 1.
\eeqn
Then with probability at least $1-\delta$,
\beq
\P\set{(x,y)\in\X\times\set{-1,1}:\sign(f(x))\neq y}
&\le&
\frac2n\paren{
D\log_2(34en/D)\log_2(578n)+\log_2(4/\delta)
}.
\eeq
\end{theorem}

In the case of hyperplanes,
the strong separation condition
(\ref{eq:strong-sep}) corresponds
to a margin of $\gamma=1/\nrm{\vec w}$.
All of our discussion of margins, envelopes, and hyperplanes continues to hold verbatim if, but rather than normalizing $\nrm{w}=1$, we normalize $\gamma=1$ and let $\nrm{w}$ vary.
It is well-known \citep[Theorem 4.6]{299098} that {\em homogeneous} hyperplanes
---
i.e., function classes of the form $\F=\set{\vec x\mapsto 
\vec w\cdot \vec x; 
\nrm{\vec x}\le 1,
\nrm{\vec w}\le R
}$
---
satisfy $\fat_\tscale(\F)\le (R/\tscale)^2$.
For this paper, we must necessarily use inhomogeneous hyperplanes,
since a general polyhedron is characterized as an intersection
of such objects.
To be more precise, the hyperplane $(\vec w,b)$
induces the function
$
f_{\vec w,b}
:\R^d\to\R$ given by $f_{\vec w,b}(\vec x)=
\vec w\cdot\vec x+b
$. Fix some $R>0$ 
and define $\F=\set{f_{\vec w,b};
\nrm{\vec w}\le R
}$.
\paragraph{Remark.}
A careful reader will notice that the conference
version of this paper,
\citet{DBLP:conf/nips/GottliebKKN18},
did not invoke fat-shattering
dimension. 
Instead, the generalization bound given there,
contained in Lemmas 1 and 2,
relied on \citet[Lemma 6]{han-kon-2017}.
Unfortunately, the latter has a mistake,
as explained 
in \citet{KA-max-fat};
hence we follow their approach
instead.
\begin{theorem}
\label{thm:poly-fat}
For $t\in\N$, let $\calP_t$ be the collection of
$t$-polytopes, i.e.,
functions $f:\R^d\to\R$ defined by
$t$ hyperplanes $(\vec w_1,b_1),\ldots,(\vec w_t,b_t)$:
\beqn
\label{eq:poly-def}
f(\vec x) = \min_{i\in[t]} \vec w_i\cdot x+b_i,
\qquad
\nrm{\vec w_i}
,
|b_i|
\le R.
\eeqn
Then
\beq
\fat_\tscale(\calP_t)
&\le&
C
t\log(t)\min\set{
\frac{R^2}{\theta^2}
,
d
}
,
\qquad
0<\tscale\le
R
,
t>1,
\eeq
where $C>0$ is a universal constant.
\end{theorem}
\begin{proof}
Observe that $
\fat_\tscale(\F)
=
\fat_\tscale(-\F)
$
and so replacing $\min$
by $\max$ in
(\ref{eq:poly-def})
will not affect $\fat_\tscale(\calP_t)$.
The result then follows directly from
\citet[Theorems 2 and 3]{KA-max-fat}.
\end{proof}

\paragraph{Dimension reduction.}
The Johnson-Lindenstrauss (JL) transform 
\citep{johnson-lindenstrauss82}
takes a set $S$ of $m$ vectors in 
$\R^d$ and projects them into 
$k = O( \eps^{-2}\log m)$
dimensions, while preserving all inter-point distances and vector
norms up to $1+\eps$ distortion. That is, if 
$f:\R^d \rightarrow \R^k$
is a linear embedding realizing the guarantees of the JL transform on $S$, 
then for every $\vec x \in S$ we have
$$(1-\eps)\|\vec x\| 
\le  \| f(\vec x) \| 
\le (1+\eps)\|\vec x\|,$$
and for every $\vec x,\vec y \in S$ we have
$$(1-\eps)\|{\vec x- \vec y}\|
\le  \| f({\vec x- \vec y}) \|
\le (1+\eps)\|{\vec x- \vec y}\|.$$
The JL transform can be realized with probability 
$1-n^{-c}$ for any constant $c \ge 1$
by a randomized linear embedding, for example a projection matrix
with entries drawn from a normal distribution
\citep{DBLP:journals/jcss/Achlioptas03}.
This embedding is {\em oblivious}, in the sense that the matrix can
be chosen without knowledge of the set $S$.

It is an easy matter to show that the JL transform can also be used to
approximately preserve distances to hyperplanes, as in the following
lemma. 

\begin{lemma}\label{lem:jl}
Let $S$ be set of $d$-dimensional vectors in the unit ball, 
$T$ be a set of normalized vectors, and 
$f:\R^d \rightarrow \R^k$
a linear embedding realizing the guarantees of the JL transform for $S \cup T$.
Then for any 
$0 < \eps < 1$
and some
$k = O((\log |S \cup T|)/\eps^2)$, 
with probability $1-|S \cup T|^{-c}$ (for any constant $c>1$)
we have for all $\vec x \in S$ and $\vec t \in T$ that
$$
f(\vec t) \cdot f(\vec x) \in \vec t \cdot \vec x \pm \eps.
$$
\end{lemma}

\begin{proof}
Let the constant in $k$ be chosen so that the JL transform preserves distances
and norms among $S \cup T$ within a factor $1+\eps'$ for
$\eps' = \eps/5$.
By the guarantees of the JL transform for the chosen value of $k$, we have that
\begin{eqnarray*}
f(\vec t) \cdot f(\vec x)
	&=& \frac{1}{2}  \bigl[\|f(\vec t)\|^2 + \|f(\vec x)\|^2 - \| f(\vec t) - f(\vec x) \|^2\bigr]	\\
	&\le& \frac{1}{2}
		\bigl[ (1+\eps')^2(\|\vec t\|^2 + \|\vec x\|^2) 
		- (1-\eps')^2 \| {\vec t - \vec x} \|^2\bigr]		\\
	&<& \frac{1}{2}
		\bigl[ (1+3\eps')(\|\vec t\|^2 + \|\vec x\|^2) 
		- (1-2\eps') \| {\vec t - \vec x} \|^2\bigr]		\\
	&<& \frac{1}{2}
		\bigl[5\eps' (\|\vec t\|^2 + \|\vec x\|^2)
		+  {2\vec t \cdot \vec x}\bigr]				\\
	&\le& 5\eps' + {\vec t \cdot \vec x}.		\\
	&=& \eps + {\vec t \cdot \vec x}.	
\end{eqnarray*}
A similar argument gives that
$f(\vec t) \cdot f(\vec x) > -\eps + {\vec t \cdot \vec x}$.
\end{proof}

\section{Computing and learning separating polyhedra}
\label{sec:adi}

In this section, we present algorithms to compute and learn 
$\gamma$-fat $t$-polyhedra. We prove hardness results for this 
problem, 
and then present our algorithms.

\subsection{Hardness}
We show that computing separating polyhedra is $\mathrm{NP}$-hard, and 
even hard to approximate. Further, we will use the 
Exponential Time Hypothesis (ETH) to justify near-exponential 
dependence on $t \gamma^{-2}$ for exact algorithms for this problem.

A consequence of ETH is that the maximum independent set
and minimum graph coloring problems cannot be solved in fewer than $c^n$ operations,
for some constant $c$ \citep{CFKLMPPS-15}.
This does not necessary imply that {\em approximating}
these problems requires $c^n$ operations: As hardness-of-approximation 
results utilize polynomial-time reductions, ETH implies only that the
runtime is exponential in some polynomial in $n$.
However, \cite{BEKP-15} have shown that under ETH, neither independent set nor
coloring admit an $r$-approximation in time 
$O(2^{n^{1-\delta}})$ for any constant $r$ and $\delta > 0$
(see \cite{BCLNN-19} for upper bounds).

We begin by considering the case of a single hyperplane, as this task is a basic tool
of the algorithms of Section \ref{sec:algorithms}.  
The following preliminary lemma builds upon 
\citet[Theorem 10]{AMALDI1995181}.

\begin{lemma}\label{lem:is-reduction}
Given a labelled point set $S$ ($n=|S|$)
with $p$ negative points, let $h^*$ be a hyperplane that places 
all positive points of $S$ on its positive side, and maximizes the number
of negative points on its negative size --- let $\opt$ be the number of these
negative points. Then
\begin{enumerate}
\item
It is $\mathrm{NP}$-hard to find a hyperplane $\tilde{h}$ consistent with all positive
points, and which places at least $\opt/p^{1-o(1)}$ negative points on
the negative side of $\tilde{h}$. This holds when the optimal hyperplane has 
margin $\gamma = \frac{1}{4 \sqrt{\opt}}$ or less with respect to the correctly
classified points.
\item
Under ETH, a hyperplane consistent with all positive points and with $\opt$ negative 
points cannot be computed in time less than $c^{1/(16\gamma^2)}$ for some constant $c$.
A hyperplane consistent with all positive points and with at least $r \opt$ negative 
points cannot be computed in time  
$O \left( 2^{(1/(16\gamma^2))^{1-\delta}} \right)$
for any constants $0 < r,\delta < 1$.
\end{enumerate}
\end{lemma}

\begin{proof}
We reduce from maximum independent set, which for $p$ vertices 
is hard to approximate to
within $p^{1-o(1)}$ \citep{Zuckerman07}. Given a graph $G=(V,E)$, 
for each vetex $v_i \in V$ place a negative point on the basis vector
$\vec e_i$. Now place a positive point at the origin, and for each 
edge $(v_i,v_j) \in E$, place a positive point at $(\vec e_i+\vec e_j)/2$.

Consider a hyperplane consistent with the positive points and placing $\opt$ 
negative points on the negative side: These negative points must represent
an independent set in $G$, for if $(v_i,v_j) \in E$, then by construction
the midpoint of $\vec e_i, \vec e_j$ is positive, and so both 
$\vec e_i, \vec e_j$ cannot lie
on the negative side of the hyperplane.

Likewise, if $G$ contained an independent set $V' \subset V$ of size $\opt$, 
then we consider the hyperplane defined by the equation 
$\vec w \cdot \vec x + \frac{3}{4\sqrt{\opt}} = 0$,
where coordinate 
$\vec w(j)=-\frac{1}{\sqrt{\opt}}$
if $\vec v_j \in V'$ 
and 
$\vec w(j)=0$
otherwise.
It is easily verified that the distance from the hyperplane to a
correctly classified negative point (i.e.\ a basis vector) is 
$-\frac{1}{\sqrt{\opt}} + \frac{3}{4\sqrt{\opt}} = - \frac{1}{4 \sqrt{\opt}}$, 
to the origin is $\frac{3}{4 \sqrt{\opt}}$, 
and to all positive points is at least 
$-\frac{1}{2\sqrt{\opt}} + \frac{3}{4\sqrt{\opt}} = \frac{1}{4 \sqrt{\opt}}$.
The first item follows.

For the second item, as the above reduction yields margin
$\gamma = \frac{1}{4 \sqrt{\opt}} \ge \frac{1}{4 \sqrt{p}}$,
by ETH we cannot compute an exact solution in time less than
$c^p \ge c^{1/(16\gamma^2)}$,
or an approximation solution in time of the order
$2^{p^{1-\delta}} \ge 2^{(1/(16\gamma^2))^{1-\delta}}$.
\end{proof}

We can now extend the above results for hyperplanes to similar ones 
for polyhedra:

\begin{theorem}\label{thm:gc-reduction}
Given a labelled point set $S$ ($n=|S|$)
with $p$ negative points, let $H^*$ be a collection of $t$ 
halfspaces whose intersection partitions $S$ into positive and negative
sets. Then 
\begin{enumerate}
\item
It is $\mathrm{NP}$-hard to find a collection $\tilde{H}$ of size less than
$tp^{1-o(1)}$ whose intersection also partitions $S$ into positive and negative sets. 
This holds when the polyhedron implied by $H^*$ has margin
$\gamma = \frac{1}{4\sqrt{p/t}}$ or less.
\item
Under ETH, a collection of $t$ halfspaces whose intersection
partitions $S$ into positive and negative sets
cannot be computed in time less than $c^{1/(16\gamma^2)}$ 
for some constant $c$, nor in time 
$O \left( 2^{(t/(16\gamma^2))^{1-\delta}} \right)$
for any constant $\delta > 0$.
\end{enumerate}
\end{theorem}

\begin{proof}
The reduction is from minimum coloring, which is hard to approximate
within a factor of $n^{1-o(1)}$ \citep{Zuckerman07}.
The construction is identical to that of the proof of 
Lemma \ref{lem:is-reduction}:
Given a graph $G=(V,E)$,
for each vetex $v_i \in V$ place a negative point on the basis vector
$\vec e_i$. Now place a positive point at the origin, and for each
edge $(v_i,v_j) \in E$, place a positive point at $(\vec e_i+\vec e_j)/2$.

As above, a hyperplane consistent with the positive points and placing $m$
negative points on the negative side corresponds to an independent set
of size $m$ in $G$, for if $(v_i,v_j) \in E$, then by construction
the midpoint of $\vec e_i, \vec e_j$ is positive, and so both
$\vec e_i, \vec e_j$ cannot lie on the negative side of the hyperplane.
Then a set of $t$ hyperplanes whose intersection partitions $S$ into 
positive and negative sets implies a $t$-coloring on $G$.

Similarly, and as above, for any independent set $V'$ of size $m$ in $G$, 
there is a hyperplane that is consistent with all positive points and
with the $m$ negative points corresponding to $V'$, while also 
achieving margin 
$\frac{1}{4 \sqrt{m}} \ge \frac{1}{4 \sqrt{p}}$
on these points.
So a $t$ coloring on $G$ implies a set of $t$ hyperplanes whose intersection 
(a polyhedron with margin at least $\frac{1}{4 \sqrt{p}}$)
partitions $S$ into positive and negative sets.
To enlarge the margin further, we stipulate that no color in the solution
represent more than $p/t$ vertices;
if a color in the optimal $t$-coloring of $G$ covers more than $p/t$ vertices,
we replace it by a minimal set of colors, each coloring no more than $p/t$ 
vertices. This increases the total number of colors to at most $2t$,
but does not affect the hardness-of-approximation result.
The first item follows.

For the second item, as the above reduction yields margin
$\gamma \ge \frac{1}{4 \sqrt{p}}$,
by ETH we cannot compute an exact solution in time less than
$c^p \ge c^{1/(16\gamma^2)}$.
Likewise, as above there exists a solution polyhedron
formed by at most $2t$ halfspaces and achieving margin
$\gamma \ge \frac{1}{4 \sqrt{p/t}}$,
but by ETH, finding any solution of $rt$ halfspaces 
for constant $r>1$
cannot be computed in time of the order
$2^{p^{1-\delta}} \ge 2^{(t/(16\gamma^2))^{1-\delta}}$.
\end{proof}

Theorem \ref{thm:gc-reduction} roughly justifies the exponential dependence 
on $t\gamma^{-2}$ in the algorithm of \citet{DBLP:journals/ml/ArriagaV06},
and implies that to avoid an exponential dependence on $t$ in the runtime, 
we should consider broader hypothesis classes, 
for example $O(t \log m)$-polyhedra.
We do this in the next section.

\subsection{Algorithms}\label{sec:algorithms}

Here we present algorithms for computing polyhedra, and use them to
give an efficient algorithm for learning polyhedra.

In what follows, we give two algorithms inspired by the work of
\citet{DBLP:journals/ml/ArriagaV06}.
Both have runtime faster than the algorithm of
\citet{DBLP:journals/ml/ArriagaV06},
and the
second is only polynomial in $t$.
The underlying idea of our algorithms is to project the points 
from their high-dimensional origin space into
a low-dimensional target space.
We can find a good halfspace in the target space using a brute-force method, 
and then identify a halfspace in the origin
space which is consistent with the halfspace of the target space. 
We identify multiple such halfspaces in the origin space,
and their intersection yields the solution polyhedron.
Crucially, we choose these halfspaces greedily, 
at each step selecting the one which maximizes the number of negative points excluded from the 
current polyhedron.

\begin{theorem}\label{thm:compute}
Given a labelled point set $S$ ($n=|S|$) for which some 
$\gamma$-fat $t$-polyhedron correctly separates the positive and negative points
(i.e., the polyhedron is {\em consistent}),
we can compute the following with high probability:
\begin{enumerate}
\item
A consistent $(\gamma/4)$-fat $t$-polyhedron in time
$n^{O(t \gamma^{-2} \log (1/\gamma))}$.
\item
A consistent $(\gamma/4)$-fat $O(t \log n)$-polyhedron in time 
$n^{O(\gamma^{-2} \log (1/\gamma))}$.
\end{enumerate}
\end{theorem}

Before proving Theorem \ref{thm:compute}, we will need a preliminary lemma:

\begin{lemma}\label{lem:net}
Given any $0<\delta<1$, there exists a set $V$ of 
unit vectors of size
$|V| = \delta^{-O(d)}$
with the following property:
For any unit vector $\vec w$, there exists some $\vec v \in V$ that satisfies 
$\vec v \cdot \vec x \in \vec w \cdot \vec x \pm \delta$
for all vectors $\vec x$ with $\|\vec x\| \le 1$.
The set $V$ can be constructed in time
$\delta^{-O(d)}$ with high probability.
\end{lemma}

This implies that if a set $S$ admits a hyperplane $(\vec w, b)$ 
with margin $\gamma$, then $S$ admits a hyperplane $(\vec v, b)$ 
(for $\vec v \in V$)
with margin at least $\gamma-\delta$.

\begin{proof}
We take $V$ to be a $\delta$-net of the unit ball, a set satisfying
that every point on the ball is within distance $\delta$ of some
point in $V$. Then 
$|V| \le (1+2/\delta)^d$
\citep[Lemma 5.2]{DBLP:journals/corr/abs-1011-3027}.
For any unit vector $\vec w$
we have for some $\vec v \in V$ 
that $\|\vec w- \vec v\| \le \delta$,
and so for any
vector $\vec x$ satisfying $\|\vec x\| \le 1$ we have 
$$|\vec w \cdot \vec x - \vec v \cdot \vec x| 
= |(\vec w-\vec v) \cdot \vec x| 
\le \| \vec w- \vec v \| 
\le \delta.$$

The net can be constructed by a randomized greedy algorithm.
By coupon-collector analysis, it suffices to sample 
$O(|V|\log |V|)$ random unit vectors. For example, each can be chosen
by sampling its coordinate from $N(0,1)$ (the standard normal distribution),
and then normalizing the vector. The resulting set contains 
within it a $\delta$-net.
\end{proof}

\paragraph{Proof of Theorem~\ref{thm:compute}}
We first apply the Johnson-Lindenstrauss transform to 
reduce the dimension of the points in $S$ to 
$k = O(\gamma^{-2}\log (n+t))
   = O(\gamma^{-2}\log n)$
while achieving the guarantees of Lemma \ref{lem:jl} 
for the points of $S$ and the $t$ halfspaces forming the optimal
$\gamma$-fat $t$-polyhedron,
with parameter $\eps = \frac{\gamma}{24}$.
This means that in the embedded space, each halfspace 
vector $\vec w$ is embedded into a vector which is not necessarily a unit vector, but is within distance $\frac{\gamma}{24}$ of some unit vector in the embedded space, and so $\|f(\vec w)\| \in 1 \pm \frac{\gamma}{24}$.
Each halfspace vector $\vec w$ also satisfies 
$f(\vec w) \cdot f(\vec x) \in \vec w \cdot \vec x \pm \frac{\gamma}{24}$
for all $x \in S$.

Now in the embedded space we extract a $\delta$-net $V$ of unit vectors by applying Lemma \ref{lem:net} with parameter
$\delta = \frac{\gamma}{12}$, 
and since the dimension of the embedded space is $k$, 
we have $|V|=\delta^{-O(k)}$.
It follows that each halfspace vector $\vec w$ is embedded to a point $f(\vec w)$ within distance 
$\eps + \delta =
\frac{\gamma}{24} + \frac{\gamma}{12} = \frac{\gamma}{8}$
of some unit vector $\vec v \in V$.
Further, the triangle inequality gives us that
$\|\vec v- f(\vec x)\| 
\in \|f(\vec w) - f(\vec x)\| \pm \|f(\vec w) - \vec v\|
\in \|f(\vec w) - f(\vec x)\| \pm \frac{\gamma}{8}$,
and so for any $\vec x \in S$ we have
$\vec v \cdot f(\vec x) 
= \frac{\|\vec v\|+\|f(\vec x)\|-\|\vec v- f(\vec x)\|}{2}
\in \frac{\|f(\vec w)\| \pm \frac{\gamma}{24} +\|f(\vec x)\|-\|f(\vec w) - f(\vec x)\| \pm \frac{\gamma}{8}}{2}
\in f(\vec w) \cdot f(\vec x) \pm \frac{\gamma}{12}.
$

Now define the set $B$ consisting of all values of the form 
$\frac{\gamma i}{6}$ 
for integer
$i = \{0,1,\ldots,\lfloor 6/\gamma \rfloor \}$.
It follows that for each $d$-dimensional halfspace 
$(\vec w,b)$ forming the original 
$\gamma$-fat $t$-polyhedron, there
is a $k$-dimensional halfspace $(\vec v,b')$ with 
$\vec v \in V$ and $b' \in B$ satisfying 
$\vec v \cdot f(\vec x) + b' 
\in 
\vec w \cdot \vec x \pm \frac{\gamma}{12} + b \pm \frac{\gamma}{6}
= \vec w \cdot \vec x + b \pm \frac{\gamma}{4}$
for every $\vec x \in S$.
Given $(\vec v, b')$, we can recover an approximation to $(\vec w, b)$ 
in the $d$-dimensional origin space thus:
Let $S' \subset S$ include all positive points in $S$, 
along with only those negative points of $S$ for which
$|\vec v \cdot f(\vec x) + b'| \ge \frac{3\gamma}{4}$,
and it follows that 
$|\vec w \cdot \vec x + b| \ge \frac{3\gamma}{4} - \frac{\gamma}{4} = \frac{\gamma}{2}$.
As $S'$ is a separable point set with margin $\Theta(\gamma)$,
we can run the Perceptron algorithm on
$S'$ in time $O(dn \gamma^{-2})$, 
and find a $d$-dimensional halfspace $\vec w'$ consistent with 
$\vec w$ on all points at distance $\frac{\gamma}{4}$ or more from $\vec w$.
We will refer to $\vec w'$ as the $d$-dimensional mirror of $\vec v$.

We compute the $d$-dimensional mirror of every vector in $V$ for
every candidate value in $B$. 
We then enumerate all possible $t$-polyhedra by taking intersections 
of all combinations of $t$ mirror halfspaces, in total time 
$$(1/\gamma)^{O(kt)}
= n^{O(t \gamma^{-2} \log (1/\gamma))},$$
and choose the best one consistent with $S$.
The first part of the theorem follows.

Better, we may give a greedy algorithm with a much improved runtime:
First note that as the intersection of $t$ halfspaces correctly 
classifies all points, the best halfspace among them correctly 
classifies at least a $(1/t)$-fraction of the negative points with margin 
$\gamma$.
Hence it suffices to find the $d$-dimensional mirror which 
is consistent with all positive points and maximizes the number
of correct negative points, all with margin $\frac{\gamma}{4}$.
We choose this halfspace, remove from $S$
the correctly classified negative points,
and iteratively search for the next best halfspace. After
$c t \log n$
iterations (for an appropriate constant $c$), the number of remaining points is 
$$n (1-\Omega(1/t))^{c t \log n}
< n e^{-\ln n}
= 1,$$
and the algorithm terminates.
\hfill$\blacksquare$\bigskip

Having given an algorithm to {\em compute} $\gamma$-fat $t$-polyhedra, 
we can now give an efficient algorithm to {\em learn} 
$\gamma$-fat $t$-polyhedra.
We sample $m$ points, and use the second item of 
Theorem \ref{thm:compute} to find a 
$(\gamma/4)$-fat $O(t \log m)$-polyhedron
consistent with the sample. 
By Theorem~\ref{thm:poly-fat}, the class of 
$\gamma$-fat $t$-polyhedra
has 
fat-shattering dimension
at scale $\tscale=1/16$
of order
$D=O(\gamma^{-2} t \log t)$.
Choosing the size of $m$ according to Theorem~\ref{thm:bst},
we conclude:
\begin{theorem}
There exists an algorithm that learns 
$\gamma$-fat $t$-polyhedra with sample complexity
\begin{equation*}
m= 
O\paren{
\frac{D}{\eps}\log^2\frac1\eps
+ \log\frac{1}{\delta}
},
\end{equation*}
in time
$m^{O((1/\gamma^2) \log (1/\gamma))}$, where
$D=O(\gamma^{-2} t \log t)$
and
$\eps,\delta$ are the desired accuracy and confidence levels.
\end{theorem}

\section{Polyhedron margin and envelope}
\label{sec:nivasch}

In this section, we show that the notions of margin and envelope defined in Section~\ref{sec:prelim}
are, in general, quite distinct.
Fortunately,
when confined to the unit ball $\calX$, one can be used to
approximate the other.

Given two sets $S_1,S_2\subseteq\R^d$, their \emph{Minkowski sum} is given by $S_1+S_2=\{\vec{p}+\vec{q}:\vec{p}\in S_1, \vec{q}\in S_2\}$, and their \emph{Minkowski difference} is given by $S_1-S_2 = \{\vec{p}\in \R^d: \{\vec{p}\} + S_2\subseteq S_1\}$. Let $B_\gamma = \{\vec{p}\in\R^d : \|\vec{p}\|\le \gamma\}$ be a ball of radius $\gamma$ centered at the origin.

Given a polyhedron $P\in \R^d$ an a real number $\gamma>0$, let
\begin{align*}
P^{(+\gamma)} &= P + B_\gamma,\\
P^{(-\gamma)} &= P - B_\gamma.
\end{align*}
Hence, $P^{(+\gamma)}$ and $P^{(-\gamma)}$ are the results of expanding or contracting, in a certain sense, the polyhedron $P$.

Also, let $P^{[+\gamma]}$ be the result of moving each halfspace defining a facet of $P$ outwards by distance $\gamma$, and similarly, let $P^{[-\gamma]}$ be the result of moving each such halfspace inwards by distance $\gamma$. Put differently, we can think of the halfspaces defining the facets of $P$ as moving outwards at unit speed, so $P$ expands with time. Then $P^{[\pm \gamma]}$ is $P$ at time $\pm \gamma$.

Hence, the $\gamma$-envelope of $P$ is given by $\partial P^{(\gamma)} = P^{(+\gamma)} \setminus P^{(-\gamma)}$, and the $\gamma$-margin of $P$ is given by $\partial P^{[\gamma]} = P^{[+\gamma]}\setminus P^{[-\gamma]}$. See Figure~\ref{fig_margin_env}.

\begin{lemma}\label{lemma_inner}
We have $P^{(-\gamma)} = P^{[-\gamma]}$.
\end{lemma}

\begin{proof}
Each point in $P^{[-\gamma]}$ is at distance at least $\gamma$ from each hyperplane containing a facet of $P$, hence, it is at distance at least $\gamma$ from the boundary of $P$, so it is in $P^{(-\gamma)}$. Now, suppose for a contradiction that there exists a point $\vec{p}\in P^{(-\gamma)}\setminus P^{[-\gamma]}$. This means that, on the one hand, $\vec{p}$ is at distance at least $\gamma$ from every point in the boundary or the exterior of $P$, but, on the other hand, $\vec{p}$ is at distance smaller than $\gamma$ from some point $\vec{q}$ in some hyperplane that contains a facet of $P$. But such a point $\vec{q}$ lies in the boundary or the exterior of $P$. Contradiction.
\end{proof}

\begin{figure}
\centerline{\includegraphics{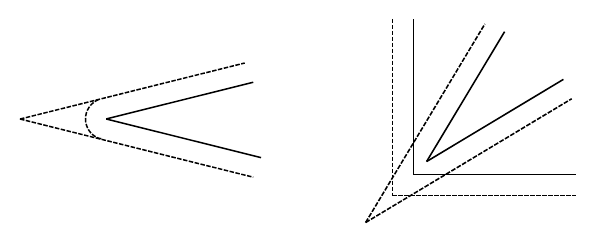}}
\caption{\label{fig_margin_bad_things}Left: $P^{[+\gamma]}$ might contain points arbitrarily far away from $P^{(+\gamma)}$; the distance becomes larger as the angle goes to $0$. Right: The $\gamma$-margin is not monotone under set containment.}
\end{figure}

However, in the other direction we have $P^{(+\gamma)} \subsetneq P^{[+\gamma]}$. Furthermore, $P^{[+\gamma]}$ might contain points arbitrarily far away from $P^{(+\gamma)}$ (see Figure~\ref{fig_margin_bad_things}, left). Moreover, the $\gamma$-margin is not monotone under set containment: There are polyhedra $Q\subseteq P$ for which $Q^{[+\gamma]}\not\subseteq P^{[+\gamma]}$ (see Figure~\ref{fig_margin_bad_things}, right).

Since the $\gamma$-margin of $P$ is not contained in the $\gamma$-envelope of $P$, we would like to find 
some sufficient condition under which, for some $\gamma'<\gamma$, the $\gamma'$-margin of $P$ is 
contained in the $\gamma$-envelope of $P$. Our solution to this problem is given in the following 
theorem. Recall that $\calX$ is the unit ball in $\R^d$.

\begin{theorem}\label{thm:lem_in_ball}
Let $P\subset \R^d$ be a polyhedron, and let $0<\gamma<1$. Suppose that $P^{[-\gamma]}\cap \calX \neq \emptyset$. Then, within $\calX$, the $(\gamma^2/2)$-margin of $P$ is contained in the $\gamma$-envelope of $P$; meaning, $\partial P^{[\gamma^2/2]} \cap \calX \subseteq \partial P^{(\gamma)}$.
\end{theorem}

(Without the condition $P^{[-\gamma]}\cap \calX \neq \emptyset$, the theorem would be false. A counterexample is given by a polygon in the plane with a very sharp vertex---one with a very acute angle---inside the unit circle. In dimensions $3$ and larger, counterexamples are possible even if none of the dihedral angles of the polyhedron are too small. Consider for example, in dimension $3$, a polyhedron with a very sharp vertex $v$ that has an equilateral-triangle cross-section. That is, $v$ is the meeting point of three edges with very acute angles between every two edges, even though the three facets meeting at $v$ make dihedral angles close to $60$ degrees.)

The proof uses the following general observation:

\begin{lemma}\label{lemma_no_reenter}
Let $Q=Q(t)$ be an expanding polyhedron whose defining halfspaces move outwards with time, each one at its own constant speed. Let $\vec{p} = \vec{p}(t)$ be a point that moves in a straight line at constant speed. Suppose $t_1<t_2<t_3$ are such that $\vec{p}(t_1)\in Q(t_1)$ and  $\vec{p}(t_3)\in Q(t_3)$. Then $\vec{p}(t_2)\in Q(t_2)$ as well.
\end{lemma}

\begin{proof}
Otherwise, $\vec{p}$ exits one of the halfspaces and enters it again, which is impossible.
\end{proof}

\paragraph{Proof of Theorem~\ref{thm:lem_in_ball}}
By Lemma~\ref{lemma_inner}, it suffices to show that $P^{[+\gamma^2/2]}\cap \calX \subseteq P^{(+\gamma)}$. Hence, let $\vec{p}\in P^{[+\gamma^2/2]}\cap \calX$ and $\vec{q}\in P^{[-\gamma]}\cap \calX$. Let $s$ be the segment $\vec{p}\vec{q}$. Let $\vec{r}$ be the point in $s$ that is at distance $\gamma$ from $\vec{p}$. Suppose for a contradiction that $\vec{p}\notin P^{(+\gamma)}$. Then $\vec{r}\notin P$.
  Consider $P = P(t)$ as a polyhedron that expands with time, as above. Let $\vec{z}=\vec{z}(t)$ be a point that moves along $s$ at constant speed, such that $\vec{z}(-\gamma)=\vec{q}$ and $\vec{z}(\gamma^2/2) = \vec{p}$. Since $\|\vec{r}-\vec{q}\|\le 2$, the speed of $s$ is at most $2/\gamma$. Hence, between $t=0$ and $t=\gamma^2/2$, $\vec{z}$ moves distance at most $\gamma$, so $\vec{z}(0)$ is already between $\vec{r}$ and $\vec{p}$. In other words, $\vec{z}$ exits $P$ and reenters it, contradicting Lemma~\ref{lemma_no_reenter}.
\hfill$\blacksquare$\bigskip

It follows immediately from Theorems
\ref{thm:poly-fat}
and
\ref{thm:lem_in_ball} 
that the 
$\tscale$-fat-shattering
dimension
of the class of 
$d$-dimensional
$t$-polyhedra with envelope $\gamma$ is at most
\beq
O\paren{
t\log(t)
\min\set{
d,
\frac1{\gamma^4\theta^2}
}
}
.
\eeq
Likewise, we can approximate the optimal $t$-polyhedron with envelope $\gamma$ 
by the algorithms of Theorem \ref{thm:compute} (with parameter $\gamma' = \gamma^2/2$).

\subsection{Hausdorff speed of an expanding polyhedron}\label{sec:hspeed}

The above technique also yields a result of independent interest, regarding what we call the \emph{Hausdorff speed} of an expanding polyhedron.

Given a set $S\subseteq \R^d$ and a point $\vec{p}\in\R^d$, let $\delta_{\vec{p}}(S) = \inf_{\vec{q}\in S} \|\vec{p}-\vec{q}\|$ be the infimum of the distances between $\vec{p}$ and the points of $S$. If $S$ is closed and convex then there exists a unique point in $S$ that achieves the minimum distance $\delta_{\vec{p}}(S)$.

Given two sets $S_1, S_2\subseteq \R^d$, the \emph{Hausdorff distance} between them is defined as
\begin{equation*}
\haus{S_1}{S_2} = \max\left\{ \textstyle \sup_{\vec{p}\in S_1} \delta_{\vec{p}}(S_2), \sup_{\vec{p}\in S_2} \delta_{\vec{p}}(S_1) \right\}.
\end{equation*}

As above, let $P=P(t)$ be a polyhedron that expands with time, due to its bounding halfspaces moving outwards, each at its own constant speed. We define the \emph{Hausdorff speed} of $P$ at time $t$ by
\begin{equation*}
\lim_{\gamma\to 0} \haus{P(t)}{P(t+\gamma)}/\gamma.
\end{equation*}

It is easy to see that the Hausdorff speed of $P$ stays constant most of the time, changing only at those time moments at which $P$ undergoes a combinatorial change (some faces of $P$ disappear and others appear in their stead).

\begin{lemma}
At each combinatorial change, the Hausdorff speed of $P$ either stays the same or decreases.
\end{lemma}

\begin{proof}
Suppose for a contradiction that at time $t_1$, the Hausdorff speed of $P$ increases from $v_1$ to $v_2$, where $v_1<v_2$. Then there exists a small enough $\gamma>0$ such that every point of $P(t_1)$ is at distance at most $\gamma v_1$ from $P(t_1-\gamma)$, and such that $P(t_1+\gamma)$ contains a point $\vec{p}$ at distance at least $\gamma v_2$ from $P(t_1)$. Let $\vec{q}$ be the point of $P(t_1)$ that is closest to $\vec{p}$, so $\|\vec{q}-\vec{p}\|=\gamma v_2$. Let $\vec{r}$ be a point of $P(t_1-\gamma)$ at distance at most $\gamma v_1$ from $\vec{q}$. Then $\|\vec{r}-\vec{p}\|\le \gamma(v_1+v_2)$. Hence, a moving point $\vec{z}$ that starts at $\vec{r}$ at time $t_1-\gamma$ and reaches $\vec{p}$ at time $t_1+\gamma$, moving at a constant speed along a straight line, exits $P$ and reenters it, contradicting Lemma~\ref{lemma_no_reenter}.
\end{proof}

\section{Experimental results}
\label{sec:exp}

\begin{figure}
  \centerline{\includegraphics{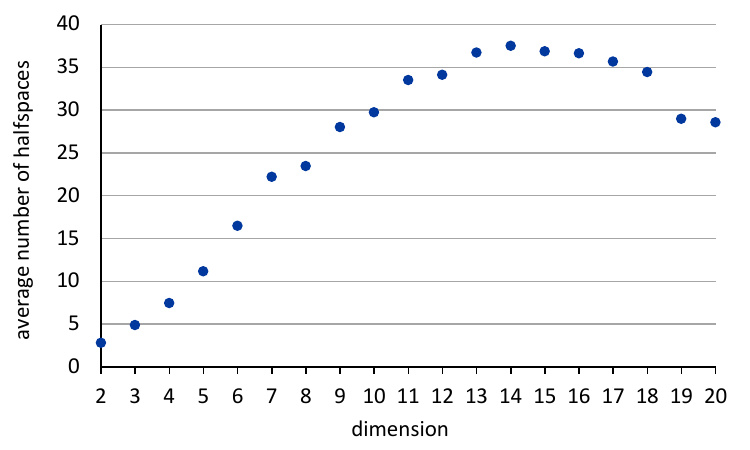}}
\caption{\label{fig:experiments} Average number of halfspaces returned by the heuristic.}
\end{figure}

Our greedy algorithmic approach motivates a simple heuristic for the 
construction of a consistent polyhedron: Given the $d$-dimensional point 
set $S$ and a run-time parameter $M$, at each iteration, we sample $M$ 
$d$-dimensional unit vectors by sampling each coordinate from the normal 
$N(0,1)$ distribution and then normalizing. Each sampled vector 
$\vec w_i$ implies a candidate direction of a $d$-dimensional halfspace, 
and we then translate the halfspace 
to distance $b_i \in [0,1]$ to the origin for the minimum $b_i$ which places all 
positive points on the inside of the halfspace (that is the side of the
origin). 
Having enumerated $M$ halfspaces of the form $(\vec w_i, b_i)$,
we select the halfspace which maximizes the number of negative points 
outside the halfspace, remove these negative points from 
$S$, and continue to the next iteration. The heuristic terminates when 
all negative points have been removed, and then outputs the chosen 
halfspaces.

As a proof of concept, we tested the above heuristic on randomly generated 
point sets and polyhedra in dimensions $d=2,\ldots,20$.
For each $d$, we first randomly sampled 1000 $d$-dimensional data vectors 
from within the unit sphere:
This is done by sampling each vector $p_i$ from the unit sphere as done above,
and then dividing the vector by $u_i^{1/d}$, where $u_i \in [0,1]$ is sampled 
randomly from the uniform distribution.\footnote{It is well-known (and easy to see) that
this method indeed uniformly samples from the unit ball:
the standard normal distribution is spherically symmetric, while the differential volume
of an infinitesimal spherical shell of radius $r$ scales as $r^d$ --- and hence taking the $d$th
root achieves uniformity.}
We then sampled $d$ halfspaces whose intersection forms the target polyhedron:
We sample a random direction vector 
$\vec w_j$
uniformly from the unit sphere, and then
sample a random offset value $b_j \in [.05,.95]$
to produce the halfspace $(\vec w_j, b_j)$.
The intersection of these $d$ halfspaces gives the target (open) polyhedron.
All data points inside the polyhedron with margin 0.05 were labelled 
as positive, all data points outside the polyhedron with margin 0.05 were 
labelled as negative, and the rest were discarded. 
We then ran the heuristic with parameter $M=10000$. Results are reported 
in Figure \ref{fig:experiments},
averaged over 100 trials per dimension.

The worst results were achieved at dimension 11, where the returned polyhedra were defined
by about three times more halfspaces than optimal. Thereafter the ratio decreases, and
by dimension 15 we actualy see a decrease in the number of halfspaces. We believe that 
this is due to known properties of high-dimensional balls: 
As the dimension grows, halfspaces that are offset away from the origin cut off a 
progressively smaller proportion of the ball's volume, and this results in many fewer 
negative points and a simpler algorithmic problem.

\section{Conclusions and future directions}

An interesting direction for future research is 
efficient learning of polyhedra defined by 
$k$ extremal points (as opposed to $k$ halfspaces).
\cite{Kupavskii20} has recently given an exponentially improved bound for the VC-dimension of these objects, but the question of learning and fat-shattering dimension are still open.

\subsubsection*{Acknowledgments}
We thank Sasho Nikolov, Bernd G\"artner and David Eppstein for helpful discussions. The 1st and 3rd authors were supported in part by
the Israel Science Foundation
(grant No. 1602/19), an Amazon Research Award,
and the BGU Data Science Research Center.

\bibliographystyle{spbasic}

\appendix

\section{Reduction of linear programming to the consistent hyperplane problem}\label{app}

The following argument is due to \citet{Nikolov18}, using techniques from~\citep{Chvatal18}.

In the consistent hyperplane problem, we are given two finite sets of points $P, Q\subset \R^d$,
and we want to find a hyperplane that strictly separates $P$ from $Q$. We show that, if there
exists a strongly polynomial time algorithm that solves the consistent hyperplane problem for the
special case $Q=\{\vec{0}\}$, then there exists a strongly polynomial time algorithm for LP.
The existence of the latter is a major open problem in complexity theory,
known as Smale's 9th problem \citep{MR1754783}.

Recall that the problem of finding a feasible solution to an LP is equivalent to the problem
of finding an optimal solution (see e.g.~\citep{MG_LP}).

\begin{lemma}\label{lemma_weak_to_strict}
Suppose there exists a strongly polynomial time algorithm $Z$ that, given a linear system of strict inequalities $A\vec{x} < \vec{b}$, finds a feasible solution if one exists. Then there exists a strongly polynomial time algorithm for LP (i.e.~for solving a system of the form $A\vec{x} \le \vec{b}$).
\end{lemma}

\begin{proof}
Given $Z$, we construct an algorithm $W$ that, given an LP $\mathcal L$, either finds a feasible solution for $\mathcal L$ or reduces $\mathcal L$ to an equivalent LP $\mathcal L'$ with fewer variables and constraints. Hence, repeated application of $W$ will produce a solution to $\mathcal L$.

Let $\mathcal L = \{A\vec{x} \le \vec{b}\}$ be the given LP, with $A\in\Q^{m\times n}$. Given a set of indices $S\subseteq \{1, \ldots, m\}$, let $\mathcal L_{(S)} = \{A_i\vec{x} < b_i : i\in S\}$ (i.e.~$\mathcal L_{(S)}$ contains $|S|$ strict inequalities and no other constraints). We say that $S$ is \emph{maximally feasible} if $\mathcal L_{(S)}$ is feasible but no $\mathcal L_{(T)}$ is feasible for $T \supsetneq S$.

It is easy to find a maximally feasible set $S$ by making at most $m$ queries to $Z$. Now we claim that, if $S$ is maximally feasible, then every solution $\vec x$ to the original system $\mathcal L$ must have equality $A_j\vec{x} = b_j$ for every $j\notin S$: Indeed, suppose for a contradiction that $A_j\vec{x} < b_j$ for some $j\notin S$. Let $\vec y$ be a feasible solution to $\mathcal L_{(S)}$. Then the convex combination $\eps \vec y + (1-\eps) \vec x$, for small enough $\eps>0$, is a feasible solution to $\mathcal L_{(S \cup \{j\})}$.

Hence, $W$ first finds a feasible set $S$. If $|S|=m$, then the solution to $\mathcal L_{(S)}$ found by $Z$ is also a solution to $\mathcal L$. Otherwise, for every $j\notin S$ we solve for a variable in the equation $A_j\vec{x} = b_j$ and substitute it into the other constraints of $\mathcal L$. This way, we produce a smaller LP $\mathcal L'=\{A'\vec x'\le b'\}$, a solution to which yields a solution to $\mathcal L$.
\end{proof}

\begin{lemma}
Suppose there exists a strongly polynomial time algorithm $Y$ that, given a set of points $P$, finds a hyperplane that strictly separates $P$ from $Q=\{\vec 0\}$, if one exists. Then there exists a strongly polynomial time algorithm $Z$ as in Lemma~\ref{lemma_weak_to_strict}.
\end{lemma}

\begin{proof}
Given a strict system $A\vec x<\vec b$ with $A\in\Q^{m\times n}$, we query $Y$ with the point set $P = \{(-A_i,b_i) : 1\le i \le m\}\cup\{(\vec 0, 1)\}$. $Y$ will return $\vec w = (\vec x', t)$ such that $\vec p\cdot\vec w>0$ for all $\vec p\in P$. In particular, $t>0$. Hence, $\vec x = \vec x'/t$ is a solution to our system. 
\end{proof}
\end{document}